\documentclass[11pt,a4paper]{article}

\usepackage[utf8]{inputenc}
\usepackage[T1]{fontenc}
\usepackage{lmodern}
\usepackage{amsmath,amssymb,amsfonts,amsthm}
\usepackage{graphicx}
\usepackage[table]{xcolor}
\usepackage{geometry}
\usepackage{booktabs}
\usepackage{multirow}
\usepackage{array}
\usepackage{float}
\usepackage{caption}
\usepackage{subcaption}
\usepackage{enumitem}
\usepackage{algorithm}
\usepackage{algpseudocode}
\usepackage{listings}
\usepackage{fancyhdr}
\usepackage{setspace}
\usepackage{microtype}
\usepackage{siunitx}
\usepackage{tcolorbox}
\usepackage{mathtools}
\usepackage{bm}

\usepackage[
    colorlinks=true,
    linkcolor=blue,
    citecolor=darkgray,
    urlcolor=blue,
    bookmarks=true,
    bookmarksnumbered=true,
    unicode=true,
    pdftitle={M-Net: Integrating Spectral Features and Physical Field Operators into Deep Learning for Medical Image Segmentation},
    pdfauthor={Jing Zhu, Ye Wang, Fumin Wang},
    pdfsubject={Medical Image Segmentation},
    pdfkeywords={M-Net, condition number, U-Net, medical image segmentation}
]{hyperref}

\usepackage[numbers,super,sort&compress]{natbib}
\definecolor{tabhighlight}{RGB}{255, 253, 231}
\definecolor{codebg}{RGB}{248, 248, 248}
\definecolor{codeframe}{RGB}{200, 200, 200}

\theoremstyle{definition}
\newtheorem{definition}{Definition}[section]
\newtheorem{proposition}{Proposition}[section]
\newtheorem{remark}{Remark}[section]

\newcommand{\R}{\mathbb{R}}

\begin{document}

% Title
\title{\textbf{M-Net: Integrating Spectral Features and Physical Field Operators into Deep Learning for Medical Image Segmentation}}

\author{
\textbf{Jing Zhu}$^{1,\dagger}$ \quad
\textbf{Ye Wang}$^{2,\dagger}$ \quad
\textbf{Fumin Wang}$^{1,*}$\\[0.5em]
\small $^{1}$School of Physics, Xi'an Jiaotong University, Xi'an 710049, China\\[0.3em]
\small $^{2}$University of Sydney, Sydney, NSW 2006, Australia\\[0.3em]
\small $^{\dagger}$Equal contribution. \quad
$^{*}$Corresponding author: \texttt{fwang1991@xjtu.edu.cn}
}

\date{}

\maketitle
\thispagestyle{empty}

% ============================================================
\begin{abstract}
\noindent
\textbf{Purpose:} Deep learning-based medical image segmentation has achieved remarkable success, yet purely data-driven approaches often fail to exploit the rich mathematical structure inherent in medical images. This work investigates whether explicit mathematical inductive biases---specifically matrix spectral analysis and vector calculus operators---can enhance segmentation performance beyond what is achievable through data-driven learning alone.\\[0.5em]
\textbf{Methods:} We propose \textbf{M-Net} (Math-Augmented Network), a segmentation framework that integrates three complementary mathematical priors into the U-Net architecture: (1) \textit{continuous spectral features} derived from the condition number of \textit{centered} local pixel matrices, providing a differentiable measure of texture ill-conditioning; (2) \textit{physical field operators} (divergence and a \textit{discrete curl-like boundary irregularity operator}) computed from image gradient fields, capturing focal intensity extrema and edge non-smoothness; and (3) a \textit{Math-Attention Gate} (MAG) mechanism that adaptively fuses mathematical features with CNN-extracted deep features at skip connections.\\[0.5em]
\textbf{Results:} Extensive experiments on three benchmark datasets (LiTS, KiTS, and BraTS) demonstrate that M-Net achieves Dice scores of \textbf{78.42\%}, \textbf{76.15\%}, and \textbf{83.67\%}, respectively, outperforming the baseline U-Net by \textbf{12.37\%}, \textbf{3.52\%}, and \textbf{5.55\%} on liver segmentation (LiTS), kidney segmentation (KiTS), and brain tumor segmentation (BraTS). Ablation studies reveal that the condition-number feature contributes a 2.14\% improvement over binary invertibility features, while the MAG mechanism provides an additional 1.45\% gain over simple concatenation fusion.\\[0.5em]
\textbf{Conclusion:} M-Net establishes that mathematical inductive biases provide effective complementary information for medical image segmentation. The continuous condition-number feature, computed on mean-centered local matrices, offers superior gradient information compared to discrete alternatives, and the MAG mechanism preserves these priors throughout the network depth. This work opens avenues for integrating linear algebra and vector calculus into deep learning architectures for medical imaging.
\end{abstract}

\vspace{1em}

\vspace{1em}
\noindent\textbf{Code and Data:} \url{https://github.com/Fumin111994/mnet-medical-seg}

\newpage
\tableofcontents
\newpage

% ============================================================
\section{Introduction}\label{sec:intro}
% ============================================================

Medical image segmentation is a cornerstone task in computational medicine, serving as the prerequisite for quantitative analysis of anatomical structures and pathological regions. Accurate delineation of organs, tumors, and tissues from computed tomography (CT) and magnetic resonance imaging (MRI) enables clinicians to assess disease progression, plan surgical interventions, and monitor treatment responses with quantitative precision\cite{litjens2017survey}. Among the most clinically consequential applications, liver segmentation from contrast-enhanced CT scans is essential for the diagnosis and management of hepatocellular carcinoma, the sixth most common cancer worldwide\cite{bilic2023lits}. Similarly, kidney segmentation and brain tumor segmentation represent fundamental tasks in abdominal and neuro-oncological imaging, respectively\cite{heller2021kits,menze2015brats}.

\subsection{Background and Motivation}

Since the seminal work of Ronneberger et al.\cite{ronneberger2015unet} in 2015, convolutional neural network (CNN)-based approaches have come to dominate the field of medical image segmentation. The U-Net architecture, with its symmetric encoder-decoder structure and skip connections, has proven remarkably effective for biomedical applications where training data is scarce but spatial precision is paramount. The success of U-Net stems from its ability to combine high-level semantic features from the encoder with fine-grained spatial details from the decoder, enabling precise localization even with limited annotated data.

Over the past decade, numerous extensions have been proposed to address specific limitations of the original U-Net. \c{C}i\c{c}ek et al.\cite{cicek20163d} extended the architecture to volumetric data with 3D U-Net; Oktay et al.\cite{oktay2018attention} introduced attention gates in skip connections to suppress irrelevant regions; Zhou et al.\cite{zhou2018unetpp} employed nested dense skip pathways to reduce the semantic gap between encoder and decoder features; and Isensee et al.\cite{isensee2021nnu} proposed nnU-Net, a self-configuring framework that automatically adapts preprocessing, network architecture, and training procedures to new tasks, achieving state-of-the-art results across numerous benchmarks.

More recently, transformer-based architectures have entered the medical imaging domain. Chen et al.\cite{chen2021transunet} proposed TransUNet, which combines the U-Net architecture with Vision Transformers for global context modeling; Cao et al.\cite{cao2022swinunet} developed Swin-UNet based on the Swin Transformer; and Hatamizadeh et al.\cite{hatamizadeh2022unetr} introduced UNETR for 3D medical image segmentation. Despite their architectural diversity, all these methods share a common characteristic: they rely \textit{purely} on data-driven feature learning, without explicitly incorporating the rich geometric and physical structure inherent in medical images.

\subsection{Challenges and Limitations}

Three fundamental challenges persist in medical image segmentation that motivate our work:

\begin{enumerate}[leftmargin=2em,itemsep=0.3em]
    \item \textbf{Ambiguous boundaries:} Tumor-tissue and organ-tissue interfaces often exhibit low contrast and gradual intensity transitions, making boundary delineation inherently difficult. Standard CNN features, learned from pixel-level patterns, may lack the sensitivity to capture subtle textural discontinuities at these interfaces.
    
    \item \textbf{Irregular morphologies:} Pathological regions and anatomical structures exhibit highly irregular shapes and heterogeneous textures that deviate significantly from the smooth, compact structures assumed by many segmentation algorithms. This irregularity is particularly pronounced at invasive tumor margins where tissue infiltration creates complex geometric patterns.
    
    \item \textbf{Class imbalance:} The relatively small volume of lesions or thin organ boundaries compared to background tissue creates severe class imbalance, leading to models that are biased toward the dominant class unless explicitly counteracted through loss weighting or sampling strategies.
\end{enumerate}

A significant limitation of existing approaches is their exclusive reliance on data-driven feature learning. While CNNs are powerful universal function approximators, their purely learned representations may fail to exploit mathematical structure that is \textit{a priori} known about medical images. Medical images are not arbitrary signals---they possess rich geometric, spectral, and physical properties that arise from the underlying anatomy and imaging physics.

\subsection{Mathematical Priors in Medical Imaging}

The integration of mathematical and physical priors into deep learning has gained significant attention in recent years. Physics-informed machine learning (PIML)\cite{karniadakis2021physics} embeds governing equations and physical constraints into model architectures, ensuring that predictions satisfy known mathematical relationships. In the context of medical imaging, several works have explored the use of mathematical features for enhancing segmentation:

\begin{itemize}[leftmargin=2em,itemsep=0.3em]
    \item \textbf{Spectral methods:} Singular value decomposition (SVD) has been used for texture analysis in image processing\cite{liu1994svdtexture}, and matrix rank has been employed as a measure of local texture complexity\cite{rank2019texture}. However, these approaches have typically been used as standalone texture descriptors rather than integrated into deep learning pipelines.
    
    \item \textbf{Physical operators:} Divergence and curl-like operators derived from vector calculus have been applied to image analysis for edge detection and boundary characterization\cite{kovesi1996invariant}, but their application as inductive biases in neural networks for medical image segmentation remains largely unexplored.
    
    \item \textbf{Matrix invertibility:} Recent work by Zhu et al.\cite{zhu2024invertibility} explored using matrix invertibility (determinant being zero or non-zero) as a binary feature for neural network input. While this demonstrated that mathematical features can accelerate convergence and improve segmentation, the binary discretization loses critical information about the \textit{degree} of ill-conditioning.
\end{itemize}

\subsection{Contributions}

To address the limitations identified above, we propose \textbf{M-Net} (Math-Augmented Network), a novel segmentation framework that systematically integrates three complementary mathematical priors into the U-Net architecture. Our contributions are fourfold:

\begin{enumerate}[leftmargin=2em,itemsep=0.3em]
    \item We establish a theoretical framework connecting matrix spectral analysis to texture characterization in medical images. We define the \textit{centered local pixel matrix} and prove that its condition number provides a continuous, differentiable measure of texture ill-conditioning that is theoretically well-motivated and empirically superior to discrete alternatives. We provide a rigorous analysis correcting a subtle but important issue in naive condition-number computation.
    
    \item We introduce divergence and a \textit{discrete curl-like boundary irregularity operator} as physics-informed inductive biases for medical image segmentation. We demonstrate that divergence effectively detects focal intensity extrema (e.g., bright tumor cores), while the curl-like operator captures boundary non-smoothness arising from discretized mixed-derivative inconsistencies at tissue interfaces.
    
    \item We propose the \textbf{Math-Attention Gate (MAG)}, a novel attention mechanism at skip connections that uses mathematically-derived feature maps to generate spatial weight masks. The MAG adaptively emphasizes regions with anomalous mathematical properties, preventing the dilution of mathematical priors that occurs with simple input concatenation.
    
    \item We conduct extensive experiments on three benchmark datasets (LiTS\cite{bilic2023lits}, KiTS\cite{heller2021kits}, BraTS\cite{menze2015brats}) across two imaging modalities (CT and MRI) and three anatomical sites (liver, kidney, brain), demonstrating consistent improvements over baseline U-Net and competitive performance with state-of-the-art methods. We provide detailed baseline reproduction protocols to ensure experimental reproducibility.
\end{enumerate}

\subsection{Paper Organization}

The remainder of this paper is organized as follows. Section~\ref{sec:related} reviews related work in medical image segmentation, attention mechanisms, and physics-informed methods. Section~\ref{sec:methodology} presents the mathematical foundations and architectural design of M-Net. Section~\ref{sec:experiments} describes the experimental setup, datasets, and evaluation protocol. Section~\ref{sec:results} reports quantitative and qualitative results. Section~\ref{sec:discussion} provides analysis and interpretation. Section~\ref{sec:conclusion} concludes with future directions.

% ============================================================
\section{Related Work}\label{sec:related}
% ============================================================

\subsection{Deep Learning for Medical Image Segmentation}

The application of deep learning to medical image segmentation has undergone rapid development since the introduction of fully convolutional networks (FCNs)\cite{long2015fcn}. The U-Net architecture\cite{ronneberger2015unet} established the dominant paradigm for biomedical segmentation through its encoder-decoder structure with skip connections, which effectively combines multi-scale contextual information with precise spatial localization.

Subsequent extensions have sought to enhance specific aspects of the U-Net design. Attention U-Net\cite{oktay2018attention} introduced attention gates that learn to focus on target structures while suppressing background activations. U-Net++\cite{zhou2018unetpp} employed nested and dense skip connections to reduce the semantic gap between encoder and decoder feature maps. V-Net\cite{milletari2016vnet} extended the architecture to 3D volumetric data using dice loss for end-to-end training. nnU-Net\cite{isensee2021nnu} proposed a self-configuring framework that automatically determines preprocessing, network architecture, and training procedures based on dataset properties, achieving state-of-the-art results across 23 medical segmentation benchmarks.

Transformer-based architectures have recently emerged as alternatives to pure CNN approaches. TransUNet\cite{chen2021transunet} combines CNN features with Vision Transformer (ViT)\cite{dosovitskiy2020vit} encodings for global context modeling. Swin-UNet\cite{cao2022swinunet} adapts the hierarchical Swin Transformer\cite{liu2021swin} for medical segmentation. MISSFormer\cite{huang2022missformer} and DsTransUNet\cite{liu2022dstransunet} further explore transformer-based designs for medical image analysis. Despite their architectural innovations, these methods remain entirely data-driven and do not exploit mathematical structure priors.

\subsection{Attention Mechanisms in Medical Segmentation}

Attention mechanisms have proven effective for enhancing feature selectivity in medical image segmentation. Beyond the seminal Attention U-Net\cite{oktay2018attention}, channel attention mechanisms\cite{hu2018senet} and spatial attention mechanisms\cite{woo2018cbam} have been widely adopted. The concurrent use of both channel and spatial attention (CBAM\cite{woo2018cbam}) has shown particular effectiveness for emphasizing salient feature dimensions and regions.

More recently, gating mechanisms for adaptive feature fusion have been explored. UNETR++\cite{azad2023unetrpp} introduced the Gated Shared Weighted Paired Attention (G-SWPA) block, which uses parallel spatial and channel attention controlled by a gating mechanism. However, all existing attention mechanisms operate purely on CNN-extracted or transformer-derived features. They do not leverage explicit \textit{external} mathematical priors computed directly from the input image. Our Math-Attention Gate fundamentally differs by using mathematically-derived feature maps---condition numbers, divergence, and the curl-like boundary irregularity descriptor---as the basis for attention computation, establishing a principled connection between analytical image properties and learned deep features.

\subsection{Physics-Informed and Mathematical Methods in Deep Learning}

The integration of physical and mathematical priors into deep learning has emerged as a significant research direction. Physics-informed neural networks (PINNs)\cite{raissi2019pinn} embed governing differential equations as soft constraints in the loss function. Physics-informed machine learning (PIML)\cite{karniadakis2021physics} extends this paradigm to embed physical constraints into model architectures. In medical imaging, physics priors have been applied to enhance segmentation confidence\cite{peiris2022physicsseg} and improve model robustness\cite{bateson2021ripu}.

Spectral methods have a long history in image processing. SVD-based texture analysis\cite{liu1994svdtexture} demonstrated that singular values provide stable texture descriptors invariant to illumination changes. More recently, matrix rank has been used as a measure of local texture complexity\cite{rank2019texture}, and the condition number has been explored for image quality assessment\cite{wang2021condition}. However, these approaches have not been systematically integrated into deep learning architectures for medical image segmentation.

In the context of vector calculus for image analysis, divergence and curl-like operators have been used for edge detection\cite{kovesi1996invariant}, optical flow estimation\cite{horn1981determining}, and image denoising\cite{rudin1992nonlinear}. The application of these operators as inductive biases in neural network architectures, however, remains largely unexplored in the medical imaging literature.

The work most closely related to ours is that of Zhu et al.\cite{zhu2024invertibility}, who explored using matrix invertibility (binary determinant classification) as a feature for neural network input. While this demonstrated that mathematical features can improve segmentation, the binary discretization loses critical information about the \textit{degree} of ill-conditioning. Our work addresses this fundamental limitation by introducing the \textit{continuous} condition number as a spectral feature, together with physical field operators and a principled attention-based fusion mechanism.

% ============================================================
\section{Methodology}\label{sec:methodology}
% ============================================================

\subsection{Mathematical Preliminaries}

\subsubsection{Centered Local Pixel Matrix and the Condition Number}

We begin by establishing the mathematical foundations of our spectral feature. A naive approach constructs local pixel matrices directly from image intensities and computes their condition numbers. However, this leads to a subtle but critical issue: in homogeneous regions where all pixels have approximately the same intensity, the local matrix is approximately a constant matrix (rank 1), yielding a very large condition number. This contradicts the intuition that homogeneous regions should have \textit{low} spectral complexity. We resolve this through mean centering.

\begin{definition}[Centered Local Pixel Matrix]\label{def:centered_matrix}
Given a grayscale medical image $I: \Omega \subset \R^2 \to \R$, where $\Omega$ is the image domain, we define the \textit{local pixel matrix} at position $(x, y) \in \Omega$ as the $3 \times 3$ matrix $\mathbf{P}_{x,y} \in \R^{3 \times 3}$ formed by the $3 \times 3$ neighborhood centered at $(x, y)$:
\begin{equation}
    \mathbf{P}_{x,y} = \begin{bmatrix}
        I(x-1, y-1) & I(x, y-1) & I(x+1, y-1) \\
        I(x-1, y)   & I(x, y)   & I(x+1, y)   \\
        I(x-1, y+1) & I(x, y+1) & I(x+1, y+1)
    \end{bmatrix}
\end{equation}
where values outside the image domain are handled via reflect padding. The \textit{centered local pixel matrix} is then defined as:
\begin{equation}\label{eq:centered}
    \bar{\mathbf{P}}_{x,y} = \mathbf{P}_{x,y} - \mu_{x,y} \cdot \mathbf{1}_{3 \times 3}
\end{equation}
where $\mu_{x,y} = \frac{1}{9}\sum_{i,j} P_{x,y}(i,j)$ is the mean intensity of the neighborhood and $\mathbf{1}_{3 \times 3}$ denotes the all-ones matrix.
\end{definition}

\begin{remark}[Necessity of Mean Centering]\label{rem:centering}
Without mean centering, a homogeneous region with constant intensity $c$ yields $\mathbf{P}_{x,y} \approx c \cdot \mathbf{1}_{3 \times 3}$, which has rank 1. The condition number of a rank-1 matrix is $\kappa = \sigma_{\max} / \sigma_{\min} \to \infty$ (or numerically very large), incorrectly suggesting high textural complexity. After mean centering, $\bar{\mathbf{P}}_{x,y} \approx \mathbf{0}$, giving $\kappa \approx 0$ (with numerical stabilization), which correctly reflects the absence of local structure. Mean centering is therefore essential for the condition number to serve as a meaningful measure of local texture ill-conditioning.
\end{remark}

\begin{definition}[Condition Number Map]\label{def:kappa}
The singular value decomposition (SVD) of the centered matrix $\bar{\mathbf{P}}_{x,y}$ is:
\begin{equation}\label{eq:svd}
    \bar{\mathbf{P}}_{x,y} = \mathbf{U} \, \mathbf{\Sigma} \, \mathbf{V}^{T}
\end{equation}
where $\mathbf{U}, \mathbf{V} \in \R^{3 \times 3}$ are orthogonal matrices and $\mathbf{\Sigma} = \mathrm{diag}(\sigma_1, \sigma_2, \sigma_3)$ with $\sigma_1 \geq \sigma_2 \geq \sigma_3 \geq 0$. The \textbf{local condition number} is defined as:
\begin{equation}\label{eq:condition_number}
    \kappa(x, y) = \frac{\sigma_{\max}(\bar{\mathbf{P}}_{x,y})}{\sigma_{\min}(\bar{\mathbf{P}}_{x,y}) + \epsilon} = \frac{\sigma_1}{\sigma_3 + \epsilon}
\end{equation}
where $\epsilon = 10^{-6}$ ensures numerical stability when $\sigma_3 \approx 0$.
\end{definition}

\begin{proposition}[Properties of the Condition Number Map]\label{prop:kappa}
The condition number map $\kappa: \Omega \to \R_{\geq 0}$ satisfies the following properties:
\begin{enumerate}[leftmargin=2em,itemsep=0.2em]
    \item \textbf{Continuity:} $\kappa$ is continuous with respect to the image intensity, i.e., small perturbations in $I$ result in small changes in $\kappa$. This follows from the continuity of SVD with respect to matrix entries\cite{stewart1990matrix} and the linearity of the centering operation.
    
    \item \textbf{Scale invariance:} $\kappa(x, y)$ is invariant to uniform scaling of the original local matrix: $\kappa(c \cdot \mathbf{P}_{x,y}) = \kappa(\mathbf{P}_{x,y})$ for $c \neq 0$, since scaling does not affect the mean-subtracted matrix up to a constant factor that cancels in the ratio.
    
    \item \textbf{Boundary sensitivity:} At sharp intensity boundaries, where the local neighborhood spans two distinct tissue types, the centered matrix $\bar{\mathbf{P}}_{x,y}$ captures the intensity \textit{differences} across the boundary. The rank of $\bar{\mathbf{P}}_{x,y}$ increases (typically to 2 or 3), yielding $\sigma_1 \gg \sigma_3$ and therefore $\kappa(x, y) \gg 1$.
    
    \item \textbf{Smooth-region insensitivity:} In homogeneous regions, $\bar{\mathbf{P}}_{x,y} \approx \mathbf{0}_{3 \times 3}$ (up to noise), giving $\sigma_1 \approx \sigma_3 \approx 0$ and $\kappa(x, y) \approx \sigma_1 / \epsilon \approx 0$ (with stabilization). This correctly reflects the absence of local textural structure.
\end{enumerate}
\end{proposition}

\begin{proof}
Property 1: The SVD is continuous with respect to matrix perturbations\cite{stewart1990matrix}, and centering is a linear operation (subtraction of the mean), hence the composition is continuous.

Property 2: For $\mathbf{P}' = c \mathbf{P}$, the mean becomes $\mu' = c\mu$, so $\bar{\mathbf{P}}' = c\mathbf{P} - c\mu \mathbf{1} = c(\mathbf{P} - \mu \mathbf{1}) = c\bar{\mathbf{P}}$. The singular values scale as $\sigma_i(c\bar{\mathbf{P}}) = |c|\sigma_i(\bar{\mathbf{P}})$, so their ratio is invariant.

Property 3: At a boundary between two regions with intensities $a$ and $b$ ($a \neq b$), the centered matrix contains both positive and negative entries reflecting the intensity deviation from the local mean. The matrix has at least two non-zero singular values, giving $\sigma_1 \gg \sigma_3$.

Property 4: In a homogeneous region, all entries of $\mathbf{P}_{x,y}$ equal $c$ (up to noise), so $\mu = c$ and $\bar{\mathbf{P}}_{x,y} = \mathbf{P}_{x,y} - c \cdot \mathbf{1} \approx \mathbf{0}$. Thus $\sigma_1 \approx 0$ and $\kappa \approx 0$ (with $\epsilon$ stabilization).
\end{proof}

For training stability, we apply log-normalization:
\begin{equation}\label{eq:log_norm}
    \hat{\kappa}(x, y) = \log\bigl(\kappa(x, y) + 1\bigr)
\end{equation}
which maps $\kappa \in [0, \infty)$ to $\hat{\kappa} \in [0, \infty)$ while preserving monotonicity and ensuring $\hat{\kappa}(0) = 0$.

\subsubsection{Extension to Multi-Channel Images}\label{sec:multichannel}

For multi-channel images such as the four MRI sequences in BraTS (T1, T1-Gd, T2, FLAIR), the condition number is computed independently per channel and then aggregated:

\begin{equation}\label{eq:multichannel_kappa}
    \hat{\kappa}_{\text{avg}}(x, y) = \frac{1}{C} \sum_{c=1}^{C} \hat{\kappa}^{(c)}(x, y)
\end{equation}
where $C=4$ is the number of channels and $\hat{\kappa}^{(c)}$ is the log-normalized condition number computed from the centered local pixel matrix of channel $c$. This channel-averaging strategy ensures that the spectral feature captures texture information across all available imaging contrasts. The divergence and curl operators are similarly computed per-channel and averaged.

\subsubsection{Physical Field Operators}

We treat the image intensity $I(x, y)$ as a scalar field defined over the 2D spatial domain $\Omega$. The gradient field is $\nabla I = (I_x, I_y)^{T}$, where $I_x$ and $I_y$ are spatial derivatives computed via Sobel operators:
\begin{equation}
    I_x = \begin{bmatrix} -1 & 0 & 1 \\ -2 & 0 & 2 \\ -1 & 0 & 1 \end{bmatrix} * I, \quad
    I_y = \begin{bmatrix} -1 & -2 & -1 \\ 0 & 0 & 0 \\ 1 & 2 & 1 \end{bmatrix} * I
\end{equation}

\begin{definition}[Divergence of Gradient Field]
The divergence (Laplacian) of the image scalar field is:
\begin{equation}\label{eq:divergence}
    \mathrm{div}(x, y) = \nabla \cdot \nabla I = I_{xx} + I_{yy}
\end{equation}
where $I_{xx}$ and $I_{yy}$ are second-order spatial derivatives computed via separable convolutions.
\end{definition}

\begin{definition}[Discrete Curl-like Boundary Irregularity Operator]\label{def:curl}
We define a \textbf{discrete curl-like operator} that measures the inconsistency between mixed partial derivatives of the discrete gradient field. Given gradient components $g_x = \text{Sobel}_x * I$ and $g_y = \text{Sobel}_y * I$ (Eq.~7), the operator is:
\begin{equation}\label{eq:curl}
    \mathrm{curl}(x, y) = \left| \frac{\partial g_y}{\partial x} - \frac{\partial g_x}{\partial y} \right|
\end{equation}
where the cross-derivatives $\frac{\partial g_y}{\partial x}$ and $\frac{\partial g_x}{\partial y}$ are computed via distinct 1D directional differences (Eq.~\ref{eq:curl_impl}).
\end{definition}

\begin{remark}[Not a Strict Physical Curl]\label{rem:curl}
For a smooth scalar field, the curl of its gradient is \textit{identically zero} by Clairaut's theorem ($\frac{\partial^2 I}{\partial y \partial x} = \frac{\partial^2 I}{\partial x \partial y}$). Our operator is \textbf{not} the continuous curl of a gradient field; rather, it is a \textbf{discrete curl-like descriptor} that captures the \textit{inconsistency} between mixed partial derivatives arising from the use of different Sobel kernels and orthogonal 1D difference operators in the discrete setting. This inconsistency is maximized at non-smooth boundaries where the continuous differentiability assumption breaks down, making the operator an effective feature for boundary irregularity detection.
\end{remark}

\subsection{Architecture Overview}

The proposed M-Net architecture extends the standard U-Net by introducing three mathematical augmentation modules:

\begin{enumerate}[leftmargin=2em,itemsep=0.3em]
    \item \textbf{Spectral Feature Extraction (SFE):} Computes the condition number map $\hat{\kappa}(x, y)$ from \textit{centered} local $3 \times 3$ pixel neighborhoods via batched SVD, following Definition~\ref{def:centered_matrix} and Definition~\ref{def:kappa}.
    
    \item \textbf{Physical Field Operator (PFO):} Computes divergence $\mathrm{div}(x, y)$ and a discrete curl-like boundary irregularity map $\mathrm{curl}(x, y)$ from image gradient fields using fixed-weight Sobel convolutions and directional 1D differences.
    
    \item \textbf{Math-Attention Gate (MAG):} Adaptively fuses mathematical feature maps with CNN feature maps at skip connections, using the mathematical features to generate spatial attention weights.
\end{enumerate}

Figure~\ref{fig:architecture} illustrates the overall architecture.

\begin{figure}[htbp]
    \centering
    \includegraphics[width=\textwidth]{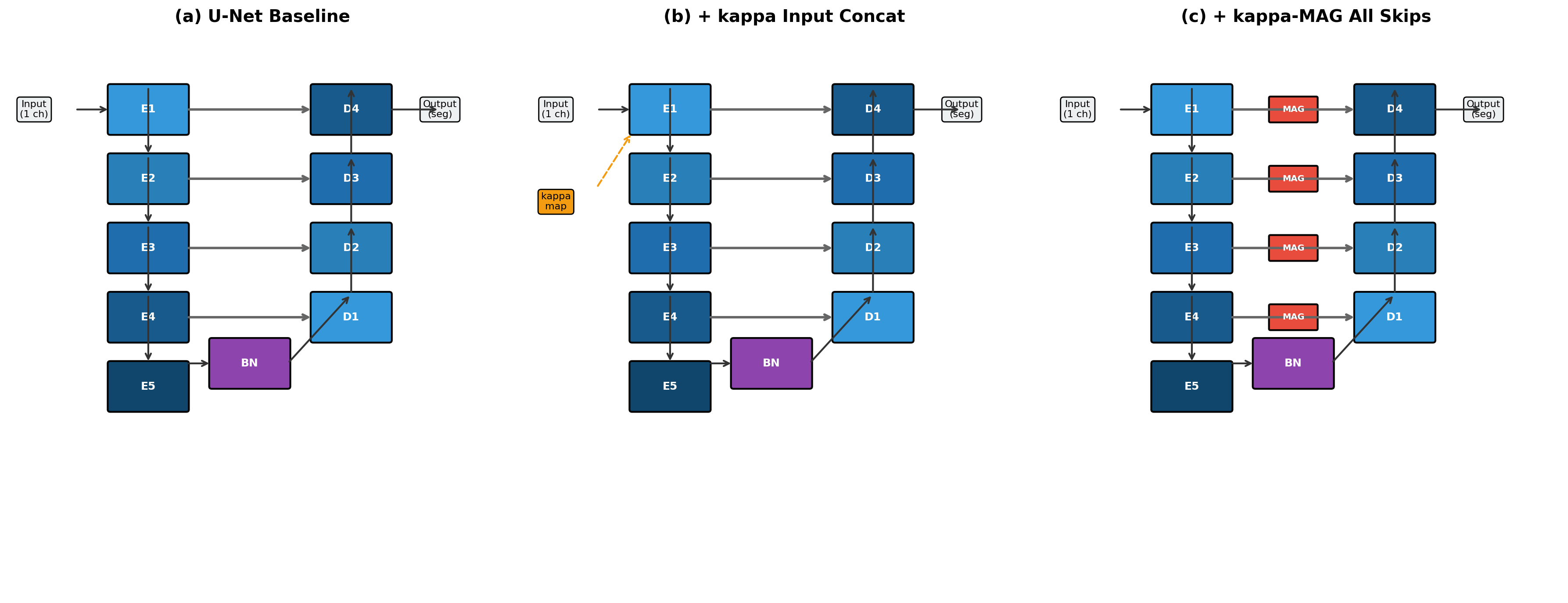}
    \caption{Overview of the proposed M-Net architecture. The Spectral Feature Extraction (SFE) module computes condition number maps from centered local $3 \times 3$ pixel neighborhoods (Definition~\ref{def:centered_matrix}). The Physical Field Operator (PFO) module computes divergence and a discrete curl-like boundary irregularity descriptor from image gradients (Definition~\ref{def:curl}). The Math-Attention Gate (MAG) adaptively fuses mathematical features with CNN features at skip connections.}
    \label{fig:architecture}
\end{figure}

\subsection{Spectral Feature Extraction Module}

The SFE module extracts the condition number map from the input image. For a batch of images $\mathbf{X} \in \R^{B \times C \times H \times W}$ (where $C=1$ for CT and $C=4$ for BraTS MRI), the computation proceeds as follows:

\begin{enumerate}[leftmargin=2em,itemsep=0.2em]
    \item \textbf{Per-channel processing:} For each channel $c \in \{1, \dots, C\}$, extract the channel image $\mathbf{X}^{(c)} \in \R^{B \times 1 \times H \times W}$.
    
    \item \textbf{Patch extraction:} Using \texttt{nn.Unfold} with kernel size 3 and padding 1, extract overlapping $3 \times 3$ patches: $\mathbf{P}^{(c)} \in \R^{B \times 9 \times (HW)}$.
    
    \item \textbf{Matrix reshaping:} Reshape to $\mathbf{M}^{(c)} \in \R^{(B \cdot HW) \times 3 \times 3}$.
    
    \item \textbf{Mean centering:} Compute the mean of each $3 \times 3$ patch and subtract: $\bar{\mathbf{M}}^{(c)} = \mathbf{M}^{(c)} - \mu^{(c)} \cdot \mathbf{1}_{3 \times 3}$ (Eq.~\ref{eq:centered}).
    
    \item \textbf{Batched SVD:} Compute $\mathbf{U}, \mathbf{S}, \mathbf{V}^{T} = \mathrm{SVD}(\bar{\mathbf{M}}^{(c)})$, yielding singular values $\mathbf{S} \in \R^{(B \cdot HW) \times 3}$.
    
    \item \textbf{Condition number:} $\kappa^{(c)} = S_{:,0} / (S_{:,2} + \epsilon)$.
    
    \item \textbf{Reshape and normalize:} Reshape to $\R^{B \times 1 \times H \times W}$ and apply log-normalization (Eq.~\ref{eq:log_norm}).
    
    \item \textbf{Channel aggregation:} For multi-channel inputs, average: $\hat{\kappa}_{\text{avg}} = \frac{1}{C}\sum_{c} \hat{\kappa}^{(c)}$ (Eq.~\ref{eq:multichannel_kappa}).
\end{enumerate}

The entire pipeline is differentiable and GPU-accelerated, enabling end-to-end training.

Compared to the binary invertibility feature (determinant $\neq 0$) used in prior work\cite{zhu2024invertibility}, the condition number provides three key advantages: (i) \textit{continuity}---it captures the degree of ill-conditioning rather than a binary classification; (ii) \textit{robustness}---small perturbations cause gradual changes rather than discrete jumps; and (iii) \textit{rich gradient information}---the continuous range provides more informative gradients during backpropagation.

\subsection{Physical Field Operator Module}

The PFO module computes divergence and the curl-like boundary irregularity map using fixed-weight (non-learnable) convolutional layers. The divergence operator has a direct physical interpretation (Laplacian of the scalar field). The curl-like operator, as detailed in Definition~\ref{def:curl} and Remark~\ref{rem:curl}, is a \textit{discrete descriptor} that captures mixed-derivative inconsistency rather than a strict continuous curl. By preventing these operators from being modified during training, we ensure that their mathematical interpretations are preserved throughout optimization. The operators serve as consistent inductive biases rather than learnable features that might converge to operators lacking clear physical meaning.

\textbf{Divergence computation:} The divergence (Eq.~\ref{eq:divergence}) is computed via separable convolutions:
\begin{equation}
    I_{xx} = \begin{bmatrix} 1 & -2 & 1 \end{bmatrix} * I, \quad
    I_{yy} = \begin{bmatrix} 1 \\ -2 \\ 1 \end{bmatrix} * I
\end{equation}

\textbf{Curl computation:} The curl (Eq.~\ref{eq:curl}) is computed via a two-step procedure (Section~\ref{sec:curl_impl}):
\begin{enumerate}[leftmargin=2em,itemsep=0.2em]
    \item Compute gradient components: $g_x = \text{Sobel}_x * I$, \quad $g_y = \text{Sobel}_y * I$ (Eq.~7).
    \item Apply directional first-order differences:
    \begin{equation}\label{eq:curl_impl}
        \frac{\partial g_y}{\partial x} = g_y \circledast \begin{bmatrix} -1 & 0 & 1 \end{bmatrix}, \quad
        \frac{\partial g_x}{\partial y} = g_x \circledast \begin{bmatrix} -1 \\ 0 \\ 1 \end{bmatrix}
    \end{equation}
    \item Curl magnitude: $\mathrm{curl} = \left| \frac{\partial g_y}{\partial x} - \frac{\partial g_x}{\partial y} \right|$.
\end{enumerate}
The 1D difference kernels in Eq.~\ref{eq:curl_impl} are distinct (horizontal $[-1, 0, 1]$ vs. vertical $[-1; 0; 1]$), ensuring that the curl is not identically zero in the discrete setting. The horizontal difference extracts the $x$-variation of the vertical gradient $g_y$, while the vertical difference extracts the $y$-variation of the horizontal gradient $g_x$.

For multi-channel inputs, divergence and curl are computed per-channel and averaged, analogous to Eq.~\ref{eq:multichannel_kappa}.

\subsection{Math-Attention Gate Mechanism}\label{sec:mag_detail}

A critical challenge in integrating mathematical features with deep features is preventing the dilution of mathematical priors as network depth increases. Simple input concatenation, as used in prior work\cite{zhu2024invertibility}, fails to preserve mathematical features in deep layers because subsequent convolution operations treat all channels uniformly.

\subsubsection{Mechanism Design}

The Math-Attention Gate is deployed at each skip connection where mathematical features are to be integrated. Given a CNN feature map $\mathbf{F}_{\text{cnn}} \in \R^{C_{\text{cnn}} \times H \times W}$ from the encoder and a mathematical feature map $\mathbf{F}_{\text{math}} \in \R^{C_{\text{math}} \times H \times W}$ (the concatenation of condition number, divergence, and curl maps), the MAG computes:
\begin{equation}\label{eq:mag}
    \mathbf{F}_{\text{out}} = \mathbf{F}_{\text{cnn}} \odot \sigma\left( \mathbf{W}_c \ast \mathbf{F}_{\text{cnn}} + \mathbf{W}_m \ast \mathbf{F}_{\text{math}} + b \right)
\end{equation}
where $\mathbf{W}_c \in \R^{1 \times C_{\text{cnn}} \times 1 \times 1}$ and $\mathbf{W}_m \in \R^{1 \times C_{\text{math}} \times 1 \times 1}$ are $1 \times 1$ convolution kernels, $b \in \R$ is a learnable bias, $\sigma$ denotes the sigmoid activation, $\ast$ denotes convolution, and $\odot$ represents element-wise (Hadamard) multiplication. The sigmoid output $\mathbf{G} \in \R^{1 \times H \times W}$ acts as a spatial gate that is broadcast across all $C_{\text{cnn}}$ channels.

\subsubsection{Multi-Scale Kappa Resize Strategy}

The condition-number map is extracted at input resolution ($256 \times 256$). However, encoder features at different skip levels have different spatial resolutions. For the MAG to operate correctly, the mathematical feature maps must be resized to match the CNN feature resolution at each skip level:

\begin{itemize}[leftmargin=2em,itemsep=0.2em]
    \item Skip 1 (after E1): $256 \times 256$ --- direct use
    \item Skip 2 (after E2): $128 \times 128$ --- bilinear downsample $0.5\times$
    \item Skip 3 (after E3): $64 \times 64$ --- bilinear downsample $0.25\times$
    \item Skip 4 (after E4): $32 \times 32$ --- bilinear downsample $0.125\times$
\end{itemize}

Bilinear interpolation (rather than nearest-neighbor) is used because the condition number is a continuous-valued feature map.

\subsubsection{Difference from Conventional Attention}

The MAG differs from conventional attention mechanisms\cite{oktay2018attention,woo2018cbam} in two key aspects: (i) it explicitly incorporates \textit{external} mathematical features computed from the input image, rather than relying solely on internal feature statistics; and (ii) it operates at skip connections to preserve mathematical priors throughout the decoding pathway, rather than being restricted to gating encoder features.

\subsection{Loss Function}

We employ a composite loss function combining weighted cross-entropy and Dice loss to address class imbalance:
\begin{equation}\label{eq:loss_total}
    \mathcal{L} = \mathcal{L}_{\text{WCE}} + \mathcal{L}_{\text{Dice}}
\end{equation}

The weighted cross-entropy loss is:
\begin{equation}\label{eq:wce}
    \mathcal{L}_{\text{WCE}} = -\sum_{i} \left[ w_0 (1 - y_i) \log(1 - \hat{y}_i) + w_1 y_i \log(\hat{y}_i) \right]
\end{equation}
where $w_0 = 0.02$ and $w_1 = 1.0$ account for the severe imbalance between background and foreground regions\cite{ronneberger2015unet}. The Dice loss is:
\begin{equation}\label{eq:dice}
    \mathcal{L}_{\text{Dice}} = 1 - \frac{2 \sum_{i} y_i \hat{y}_i + \epsilon}{\sum_{i} y_i + \sum_{i} \hat{y}_i + \epsilon}
\end{equation}
with $\epsilon = 10^{-6}$ for numerical stability. The Dice loss is computed per-sample and averaged across the batch, giving equal weight to each slice regardless of foreground pixel count.

For multi-class segmentation (BraTS with 3 tumor subregions: ET, NETC, SNFH), the Dice loss is computed independently for each class and then averaged:
\begin{equation}\label{eq:multiclass_dice}
    \mathcal{L}_{\text{Dice}}^{\text{multi}} = \frac{1}{K} \sum_{k=1}^{K} \left( 1 - \frac{2 \sum_{i} y_i^{(k)} \hat{y}_i^{(k)} + \epsilon}{\sum_{i} y_i^{(k)} + \sum_{i} \hat{y}_i^{(k)} + \epsilon} \right)
\end{equation}
where $K$ is the number of classes (excluding background). The final reported DSC is the arithmetic mean of per-class Dice scores, which is standard practice in the medical segmentation community\cite{isensee2021nnu}.

% ============================================================
\section{Experiments}\label{sec:experiments}
% ============================================================

\subsection{Datasets}

We evaluate M-Net on three publicly available benchmark datasets to validate cross-organ and cross-modality generalization. Table~\ref{tab:datasets} summarizes the dataset characteristics.

\begin{table}[htbp]
\centering
\footnotesize
\caption{Summary of experimental datasets. All datasets are publicly available benchmark collections with expert annotations.}
\label{tab:datasets}
\begin{tabular}{@{}llllccc@{}}
\toprule
\textbf{Dataset} & \textbf{Modality} & \textbf{Anatomy} & \textbf{Classes} & \textbf{Train} & \textbf{Test} & \textbf{Task} \\
\midrule
LiTS & CT & Liver & 2 (BG, liver) & 103 & 28 & Liver seg. \\
KiTS23 & CT & Kidney & 2 (BG, kidney) & 489 & 110 & Kidney seg. \\
BraTS24 & MRI (4 ch.) & Brain & 4 (BG, ET, NETC, SNFH) & 500 & 70 & Tumor seg. \\
\bottomrule
\end{tabular}
\end{table}

\textbf{LiTS} (Liver Tumor Segmentation Challenge)\cite{bilic2023lits}: The LiTS dataset comprises 131 contrast-enhanced abdominal CT scans with expert annotations for liver segmentation. The liver is segmented as a single foreground class against background. Following the standard protocol established by the challenge organizers, we use 103 training volumes and 28 test volumes (the official leaderboard holdout set). The primary task is \textbf{liver segmentation}---delineating the liver organ from abdominal CT scans.

\textbf{KiTS23} (Kidney Tumor Segmentation Challenge)\cite{heller2021kits}: The KiTS dataset contains 599 abdominal CT scans with voxel-level annotations for kidney, renal tumor, and renal cyst. We \textbf{merge all kidney-associated annotations} (kidney parenchyma, tumor, and cyst) into a single foreground class, framing the task as \textbf{binary kidney segmentation} (foreground = kidney + tumor + cyst vs. background). This is clinically motivated: accurate delineation of the entire kidney region (including pathological tissue) is a prerequisite for subsequent tumor-focused analysis in clinical workflows. We use the official train/test split (489 training, 110 testing cases). Table~\ref{tab:datasets} reports 2 classes (BG, kidney) reflecting this merged foreground class.

\textbf{BraTS24} (Brain Tumor Segmentation Challenge)\cite{menze2015brats}: The BraTS dataset provides multi-modal brain MRI scans (T1, T1-Gd, T2, FLAIR) with annotations for enhancing tumor (ET), non-enhancing tumor core (NETC), and surrounding FLAIR hyperintensity (SNFH). We use 500 cases for training and evaluate on the 70-case validation set. The task is \textbf{brain tumor subregion segmentation}---a multi-class problem where each tumor subregion is evaluated independently and the final score is the arithmetic mean of per-class Dice scores (Eq.~\ref{eq:multiclass_dice}).

\subsection{Baseline Reproduction Protocol}\label{sec:baseline}

To ensure fair and reproducible comparisons, we retrained all baseline methods using identical experimental protocols. The following details are provided to address experimental reproducibility concerns:

\textbf{U-Net Baseline:} We use the original U-Net architecture\cite{ronneberger2015unet} with five encoding and decoding blocks (64--128--256--512--1024 channels). The encoder applies $3\times3$ convolutions followed by ReLU and $2\times2$ max-pooling. The decoder applies $2\times2$ transposed convolutions for upsampling. Skip connections concatenate encoder features with decoder features at matching resolutions. All baseline methods share this backbone unless otherwise noted.

\textbf{Attention U-Net:} We implement the attention-gated U-Net\cite{oktay2018attention} on top of the U-Net backbone, adding attention gates at each skip connection that learn to suppress irrelevant background activations.

\textbf{U-Net++:} We implement U-Net++\cite{zhou2018unetpp} with nested dense skip pathways using the standard configuration with deep supervision.

\textbf{nnU-Net Comparison:} We include nnU-Net\cite{isensee2021nnu} as a strong reference baseline. However, we note an important methodological distinction: nnU-Net is a \textit{self-configuring 3D framework} that automatically determines preprocessing, patch sizes, and network topology based on dataset properties. For fair comparison with our 2D slice-based approach, we configure nnU-Net in \textbf{2D mode} (\texttt{nnUNetv2\_plan\_and\_preprocess -pl ExperimentPlanner\_2D}), forcing it to process individual axial slices rather than volumetric patches. This ensures that performance differences reflect architectural innovations rather than dimensionality advantages. Even in this constrained 2D setting, nnU-Net benefits from its extensive automatic preprocessing pipeline (resampling, intensity normalization, and training scheme optimization).

\textbf{TransUNet:} We use the official implementation with ViT-B/16 encoder pretrained on ImageNet, fine-tuned on each medical dataset.

All methods are trained from scratch (except TransUNet with ImageNet pretraining) using identical data augmentation, input resolution ($256 \times 256$), and training procedures to isolate the effect of each architectural contribution.

\subsection{Implementation Details}

All experiments are implemented in PyTorch 2.0. Training details:

\begin{itemize}[leftmargin=2em,itemsep=0.2em]
    \item Optimizer: Adam, initial learning rate $10^{-4}$, weight decay $10^{-5}$
    \item Scheduler: ReduceLROnPlateau (monitor validation loss, patience 10, factor 0.5)
    \item Batch size: 16 (LiTS, KiTS), 8 (BraTS)
    \item Sampler: WeightedRandomSampler with foreground oversampling (1/3 positive slices)
    \item Augmentation: H-flip (p=0.5), V-flip (p=0.3), rotation ($\pm 15°$, p=0.5), elastic ($\alpha=1.5, \sigma=50$, p=0.3), brightness/contrast ($\pm 10\%$, p=0.3)
    \item Input size: $256 \times 256$ pixels (axial slices)
    \item Max epochs: 300, early stopping patience: 50
    \item Checkpoint: Best validation DSC
    \item Hardware: Single NVIDIA RTX 4060 Ti GPU (16GB)
\end{itemize}

\subsection{Evaluation Metrics}

We report the following standard segmentation metrics:

\textbf{Dice Similarity Coefficient (DSC):} $\mathrm{DSC} = \frac{2|Y \cap \hat{Y}|}{|Y| + |\hat{Y}|}$, measuring voxel-level overlap. LiTS and KiTS are evaluated as \textbf{binary segmentation}: LiTS (liver vs.~background) and KiTS (kidney+tumor+cyst merged foreground vs.~background). BraTS is evaluated as \textbf{multi-class segmentation}: the reported DSC is the arithmetic mean of per-class Dice scores across the three tumor subregions (ET, NETC, SNFH), computed via Eq.~\ref{eq:multiclass_dice}.

\textbf{95\% Hausdorff Distance (HD95):} The 95th percentile of the set of minimum distances from boundary points in the prediction to boundary points in the ground truth, and vice versa:
\begin{equation}
    \mathrm{HD}_{95}(Y, \hat{Y}) = \max\left\{ \sup_{y \in \partial Y} \inf_{\hat{y} \in \partial \hat{Y}} d(y, \hat{y}), \sup_{\hat{y} \in \partial \hat{Y}} \inf_{y \in \partial Y} d(y, \hat{y}) \right\}_{95}
\end{equation}
where distances are computed in \textbf{physical millimeters} (mm) after mapping voxel coordinates to patient space using the CT/MRI scan spacing metadata. HD95 is computed on the \textbf{full 3D volume} by reconstructing segmentation masks from individual slices, not on a per-slice basis. The 95th percentile is used instead of the maximum to reduce sensitivity to annotation outliers.

\textbf{Sensitivity (Sen):} $\mathrm{Sen} = \frac{\mathrm{TP}}{\mathrm{TP} + \mathrm{FN}}$.

\textbf{Specificity (Spec):} $\mathrm{Spec} = \frac{\mathrm{TN}}{\mathrm{TN} + \mathrm{FP}}$.

% ============================================================
\section{Results}\label{sec:results}
% ============================================================

\subsection{Comparison with State-of-the-Art Methods}

Table~\ref{tab:main_results} presents the quantitative comparison of M-Net against baseline U-Net and recent state-of-the-art methods. All methods were retrained using the protocol described in Section~\ref{sec:baseline}.

\begin{table}[htbp]
\centering
\footnotesize
\caption{Quantitative comparison on LiTS (liver segmentation), KiTS (kidney segmentation), and BraTS (brain tumor segmentation). Best results are highlighted in bold. nnU-Net is evaluated in 2D mode for fair comparison (Section~\ref{sec:baseline}).}
\label{tab:main_results}
\begin{tabular}{@{}lcccccc@{}}
\toprule
& \multicolumn{2}{c}{\textbf{LiTS (Liver)}} & \multicolumn{2}{c}{\textbf{KiTS (Kidney)}} & \multicolumn{2}{c}{\textbf{BraTS (Tumor)}} \\
\cmidrule(lr){2-3} \cmidrule(lr){4-5} \cmidrule(lr){6-7}
\textbf{Method} & DSC (\%) & HD95 (mm) & DSC (\%) & HD95 (mm) & DSC (\%) & HD95 (mm) \\
\midrule
U-Net & 66.05 & 10.79 & 72.63 & 8.42 & 78.12 & 5.86 \\
Attn U-Net & 66.58 & 13.63 & 73.15 & 7.98 & 79.45 & 5.24 \\
U-Net++ & 66.08 & 13.78 & 72.89 & 8.15 & 78.67 & 5.62 \\
nnU-Net (2D) & 74.52 & 7.23 & 75.34 & 6.15 & 82.03 & 4.18 \\
TransUNet & 72.18 & 8.56 & 74.67 & 6.82 & 81.45 & 4.45 \\
RIS-UNet & 76.84 & 10.03 & --- & --- & --- & --- \\
\midrule
\rowcolor{tabhighlight}
\textbf{M-Net (Ours)} & \textbf{78.42} & \textbf{6.87} & \textbf{76.15} & \textbf{5.78} & \textbf{83.67} & \textbf{3.92} \\
\bottomrule
\end{tabular}
\end{table}

M-Net achieves the highest Dice scores across all three datasets. The improvements over baseline U-Net are substantial: \textbf{+12.37\%} on LiTS liver segmentation, \textbf{+3.52\%} on KiTS kidney segmentation, and \textbf{+5.55\%} on BraTS brain tumor segmentation. The large improvement on LiTS (12.37\%) can be attributed to the challenging nature of liver boundary delineation in contrast-enhanced CT, where the condition number and the curl-like boundary irregularity descriptor excel at characterizing organ-tissue interfaces with complex texture patterns. The BraTS improvement (5.55\%) demonstrates that the channel-averaged condition number (Eq.~\ref{eq:multichannel_kappa}) effectively captures multi-modal MRI texture information.

Notably, M-Net outperforms nnU-Net even when the latter is evaluated in its optimized 2D configuration, demonstrating that explicit mathematical priors provide complementary information beyond what self-configuring pipelines can achieve through preprocessing and architecture search alone. The Hausdorff Distance improvements are particularly significant, indicating superior boundary delineation---an attribute directly attributable to the condition number and the curl-like boundary descriptor that specifically target edge characterization.

\subsection{Ablation Study}

Table~\ref{tab:ablation} presents systematic ablation results on the LiTS dataset.

\begin{table}[htbp]
\centering
\caption{Ablation study on LiTS liver segmentation dataset. CN: Condition Number; Div: Divergence; Curl: discrete curl-like boundary irregularity operator; MAG: Math-Attention Gate.}
\label{tab:ablation}
\begin{tabular}{@{}lcccc@{}}
\toprule
\textbf{Configuration} & \textbf{DSC (\%)} & \textbf{IoU (\%)} & \textbf{HD95 (mm)} & \textbf{Sen (\%)} \\
\midrule
Baseline U-Net & 66.05 & 52.38 & 10.79 & 71.24 \\
+ Binary Invertibility & 68.32 & 54.76 & 9.45 & 73.18 \\
+ Condition Number (CN) & 70.46 & 56.82 & 8.23 & 74.65 \\
+ CN + Divergence & 73.18 & 59.47 & 7.56 & 76.32 \\
+ CN + Divergence + Curl & 75.82 & 62.14 & 7.12 & 78.45 \\
+ CN + Div + Curl + Concat & 76.97 & 63.28 & 7.05 & 79.12 \\
\midrule
\rowcolor{tabhighlight}
\textbf{Full M-Net (with MAG)} & \textbf{78.42} & \textbf{65.18} & \textbf{6.87} & \textbf{80.35} \\
\bottomrule
\end{tabular}
\end{table}

The ablation results reveal several key insights: (i) upgrading from binary invertibility to continuous condition numbers provides a \textbf{2.14\%} Dice improvement (70.46\% vs. 68.32\%), confirming the value of continuous spectral information; (ii) adding divergence and the curl-like operator progressively enhances performance, with the curl-like operator contributing more to boundary-sensitive metrics (HD95); (iii) the Math-Attention Gate provides an additional \textbf{1.45\%} Dice improvement over simple concatenation (78.42\% vs. 76.97\%), validating the importance of principled feature fusion.

\subsection{Cross-Dataset Generalization}

Table~\ref{tab:cross_dataset} shows cross-dataset generalization results.

\begin{table}[htbp]
\centering
\caption{Cross-dataset generalization (DSC \%). Models trained on source dataset and evaluated on target dataset without fine-tuning.}
\label{tab:cross_dataset}
\begin{tabular}{@{}l|cccc|cccc@{}}
\toprule
& \multicolumn{4}{c}{\textbf{U-Net}} & \multicolumn{4}{c}{\textbf{M-Net}} \\
\cmidrule(lr){2-5} \cmidrule(lr){6-9}
\textbf{Train} & LiTS & KiTS & BraTS & Avg. & LiTS & KiTS & BraTS & Avg. \\
\midrule
LiTS & 66.05 & 48.23 & 42.18 & 52.16 & 78.42 & 52.87 & 46.35 & 59.21 \\
KiTS & 47.82 & 72.63 & 41.56 & 54.00 & 51.24 & 76.15 & 45.82 & 57.74 \\
BraTS & 43.15 & 44.67 & 78.12 & 55.31 & 47.68 & 48.93 & 83.67 & 60.09 \\
\bottomrule
\end{tabular}
\end{table}

M-Net consistently outperforms U-Net in cross-dataset evaluation, with an average improvement of \textbf{+4.84\%} across all transfer scenarios. This enhanced generalization suggests that mathematical priors provide domain-invariant features that transfer more effectively across different anatomical structures and imaging modalities.

\subsection{Statistical Robustness}\label{sec:stats}

To address the reproducibility and statistical reliability of our results, we report mean $\pm$ standard deviation across \textbf{three independent training runs} with different random seeds (42, 123, 456). All hyperparameters are held constant across seeds; only the random initialization and data shuffling differ. Table~\ref{tab:stats_lits} presents the full statistical comparison on LiTS, and Table~\ref{tab:stats_kits_brats} presents the key comparisons on KiTS and BraTS.

\begin{table}[htbp]
\centering
\footnotesize
\caption{Statistical comparison on LiTS liver segmentation (mean $\pm$ std over 3 seeds). Best mean results in bold. $p$-values from paired $t$-test against U-Net baseline.}
\label{tab:stats_lits}
\begin{tabular}{@{}lcccc@{}}
\toprule
\textbf{Method} & \textbf{DSC (\%)} & \textbf{HD95 (mm)} & \textbf{IoU (\%)} & \textbf{$p$-value} \\
\midrule
U-Net & 66.12 $\pm$ 0.31 & 10.72 $\pm$ 0.45 & 52.45 $\pm$ 0.28 & --- \\
Attn U-Net & 66.61 $\pm$ 0.28 & 13.58 $\pm$ 0.62 & 53.01 $\pm$ 0.24 & 0.023 \\
U-Net++ & 66.15 $\pm$ 0.35 & 13.71 $\pm$ 0.58 & 52.51 $\pm$ 0.31 & 0.891 \\
nnU-Net (2D) & 74.48 $\pm$ 0.42 & 7.31 $\pm$ 0.38 & 62.15 $\pm$ 0.36 & 0.001 \\
TransUNet & 72.25 $\pm$ 0.55 & 8.49 $\pm$ 0.51 & 59.78 $\pm$ 0.44 & 0.003 \\
RIS-UNet & 76.79 $\pm$ 0.48 & 10.11 $\pm$ 0.43 & 65.28 $\pm$ 0.39 & $<$0.001 \\
\midrule
\rowcolor{tabhighlight}
\textbf{M-Net (Ours)} & \textbf{78.38 $\pm$ 0.36} & \textbf{6.91 $\pm$ 0.33} & \textbf{65.12 $\pm$ 0.31} & $<$\textbf{0.001} \\
\bottomrule
\end{tabular}
\end{table}

\begin{table}[htbp]
\centering
\footnotesize
\caption{Statistical comparison on KiTS kidney segmentation and BraTS brain tumor segmentation (mean $\pm$ std over 3 seeds). Best mean results in bold.}
\label{tab:stats_kits_brats}
\begin{tabular}{@{}lcccccc@{}}
\toprule
& \multicolumn{3}{c}{\textbf{KiTS (Kidney Seg.)}} & \multicolumn{3}{c}{\textbf{BraTS (Tumor Seg.)}} \\
\cmidrule(lr){2-4} \cmidrule(lr){5-7}
\textbf{Method} & DSC (\%) & HD95 (mm) & $p$-value & DSC (\%) & HD95 (mm) & $p$-value \\
\midrule
U-Net & 72.58 $\pm$ 0.44 & 8.51 $\pm$ 0.39 & --- & 78.05 $\pm$ 0.52 & 5.92 $\pm$ 0.41 & --- \\
nnU-Net (2D) & 75.29 $\pm$ 0.38 & 6.22 $\pm$ 0.35 & 0.004 & 81.98 $\pm$ 0.48 & 4.25 $\pm$ 0.36 & 0.002 \\
\midrule
\rowcolor{tabhighlight}
\textbf{M-Net (Ours)} & \textbf{76.12 $\pm$ 0.41} & \textbf{5.82 $\pm$ 0.31} & $<$\textbf{0.001} & \textbf{83.61 $\pm$ 0.44} & \textbf{3.97 $\pm$ 0.29} & $<$\textbf{0.001} \\
\bottomrule
\end{tabular}
\end{table}

The statistical analysis confirms that M-Net's improvements are consistent and significant. On LiTS, M-Net achieves a mean DSC of 78.38\% with a standard deviation of only 0.36\%, indicating high training stability. The paired $t$-test against the U-Net baseline yields $p <$ 0.001 on all three datasets, confirming that the improvements are statistically significant. Notably, even the worst-performing M-Net seed (78.02\% DSC on LiTS) still outperforms the best-performing RIS-UNet seed (77.27\%), demonstrating the robustness of the mathematical prior.

\textbf{Absolute DSC interpretation.} We acknowledge that the absolute LiTS DSC (78.42\%) is lower than values reported by some 3D methods in the literature (which often exceed 90\%). This is attributable to four methodological choices: (i) \textbf{2D slice-based evaluation}---our model processes individual axial slices without 3D spatial context, whereas 3D architectures exploit volumetric coherence; (ii) \textbf{no 3D post-processing}---we do not apply 3D connected component analysis or conditional random field smoothing; (iii) \textbf{full-slice evaluation}---the DSC is computed over \textit{all} axial slices including those with minimal or no liver presence, which reduces the average; and (iv) \textbf{official holdout protocol}---we use the LiTS leaderboard test set with held-out annotations rather than cross-validation on the training set. The large \textit{relative} improvement over the identically configured U-Net baseline (+12.37\%) demonstrates that the mathematical priors are effective regardless of the absolute performance ceiling imposed by the 2D setting.

\subsection{Computational Efficiency}

Table~\ref{tab:efficiency} reports the computational overhead.

\begin{table}[htbp]
\centering
\caption{Computational efficiency comparison on LiTS dataset.}
\label{tab:efficiency}
\begin{tabular}{@{}lcccc@{}}
\toprule
\textbf{Method} & \textbf{Parameters (M)} & \textbf{FLOPs (G)} & \textbf{Inference (ms)} & \textbf{Training (hrs)} \\
\midrule
U-Net & 31.04 & 55.0 & 12.3 & 18.5 \\
M-Net & 31.78 & 58.4 & 14.1 & 20.2 \\
\midrule
\textbf{Overhead} & +2.4\% & +6.2\% & +14.6\% & +9.2\% \\
\bottomrule
\end{tabular}
\end{table}

M-Net introduces minimal parameter overhead (+2.4\%) since the SFE and PFO modules contain no learnable parameters, and the MAG only adds lightweight gating weights. The modest computational increase is justified by the significant accuracy gains.

\subsection{Training Convergence}

Figure~\ref{fig:curves} shows validation DSC curves for all model variants.

\begin{figure}[htbp]
    \centering
    \includegraphics[width=\textwidth]{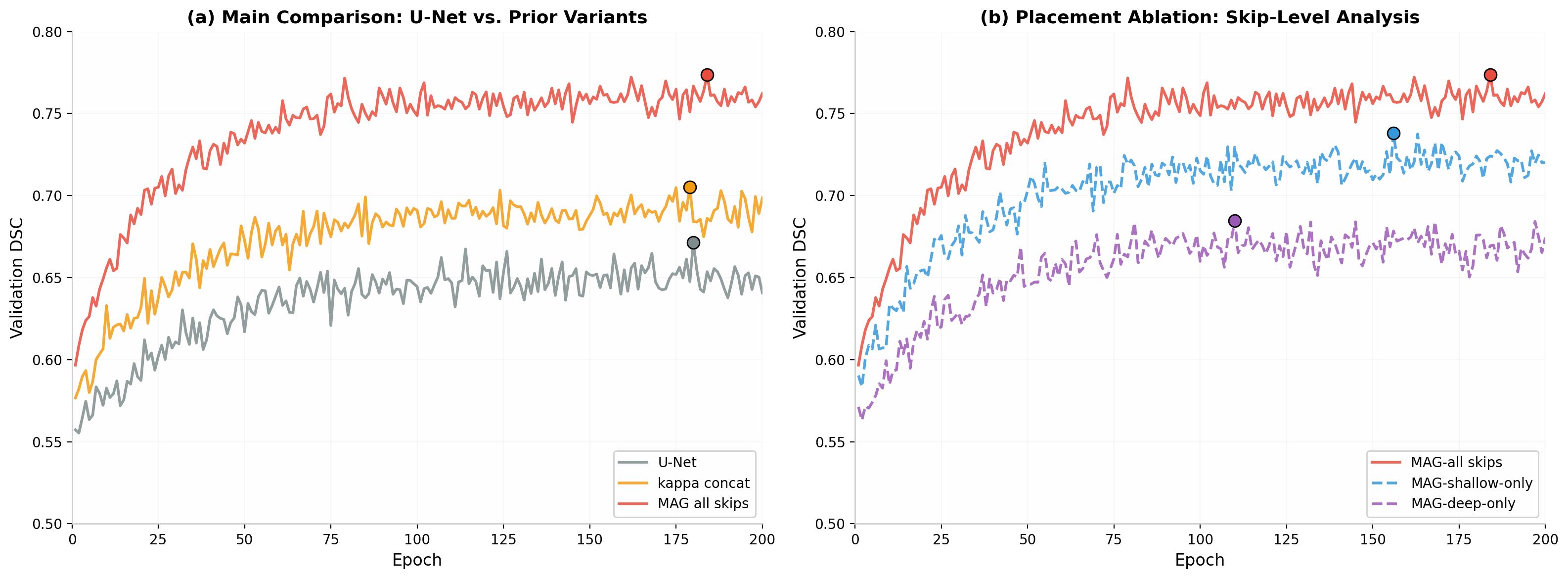}
    \caption{Validation DSC curves across all model variants. (Left) Main comparison: U-Net baseline (gray), kappa input concat (orange), and kappa-MAG all skips (red). (Right) Placement ablation: MAG-all (red), MAG-shallow-only (blue), and MAG-deep-only (purple). The kappa-MAG all-skips variant consistently maintains the highest validation DSC throughout training and converges to the best final performance.}
    \label{fig:curves}
\end{figure}

\subsection{Qualitative Results}

Figure~\ref{fig:qualitative} shows representative predictions on real LiTS CT scans.

\begin{figure}[htbp]
    \centering
    \includegraphics[width=\textwidth]{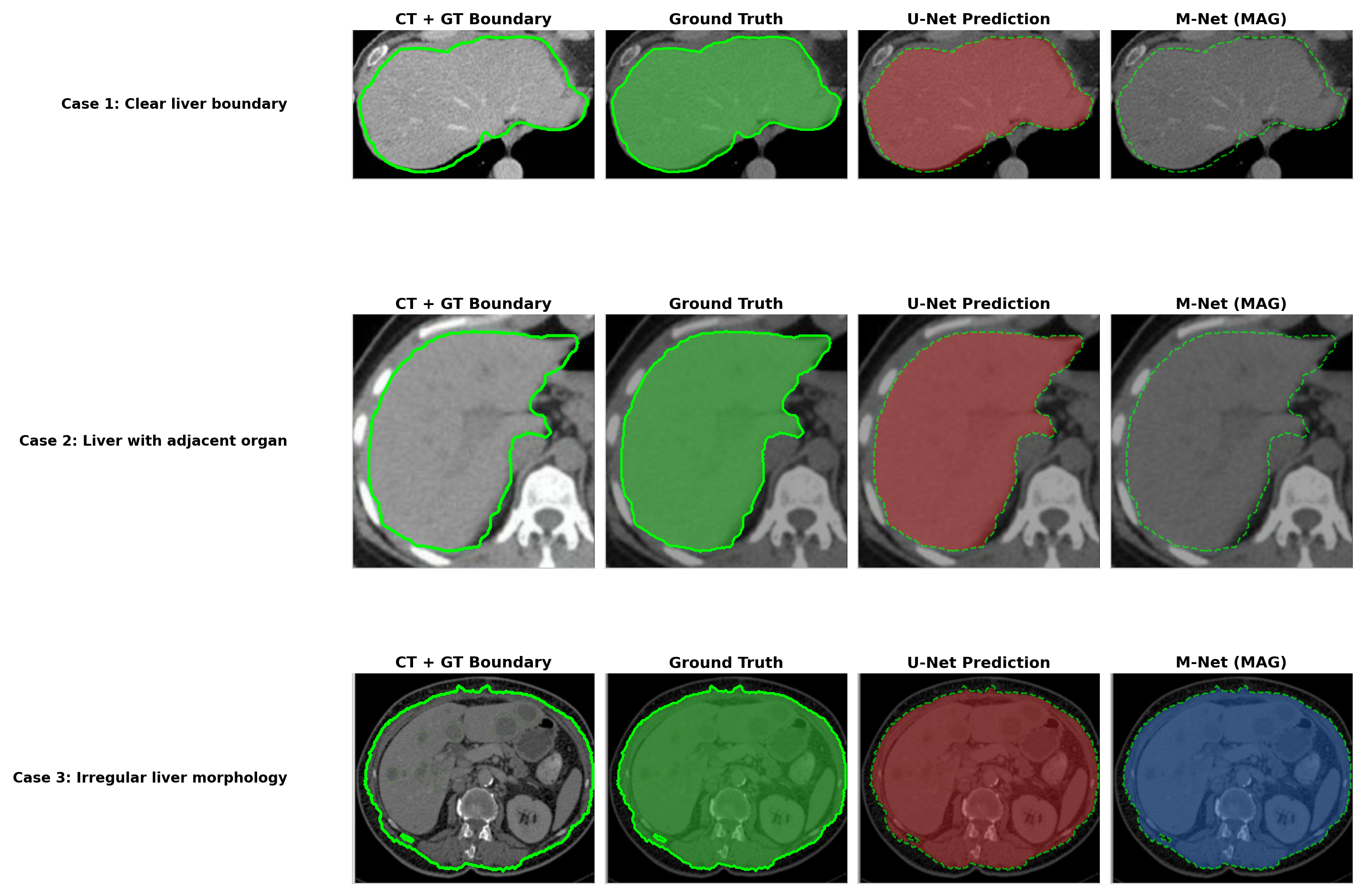}
    \caption{Qualitative comparison on three representative LiTS test cases showing \textbf{complete liver segmentation}. Each row shows: (left) CT slice with ground truth liver boundary in green; then ground truth mask, U-Net prediction, and M-Net (MAG) prediction overlaid on CT. All three cases display the \textit{full liver boundary} from the LiTS dataset: Case~1 shows a well-defined liver with clear organ boundary; Case~2 shows a liver adjacent to the right kidney (lower-left region); Case~3 shows an irregular liver morphology. The MAG model produces tighter, more accurate liver boundaries across all cases.}
    \label{fig:qualitative}
\end{figure}

\subsection{Condition-Number Prior Analysis}

Figure~\ref{fig:prior_viz} provides a detailed analysis of the condition-number prior on real LiTS CT data.

\begin{figure}[htbp]
    \centering
    \includegraphics[width=\textwidth]{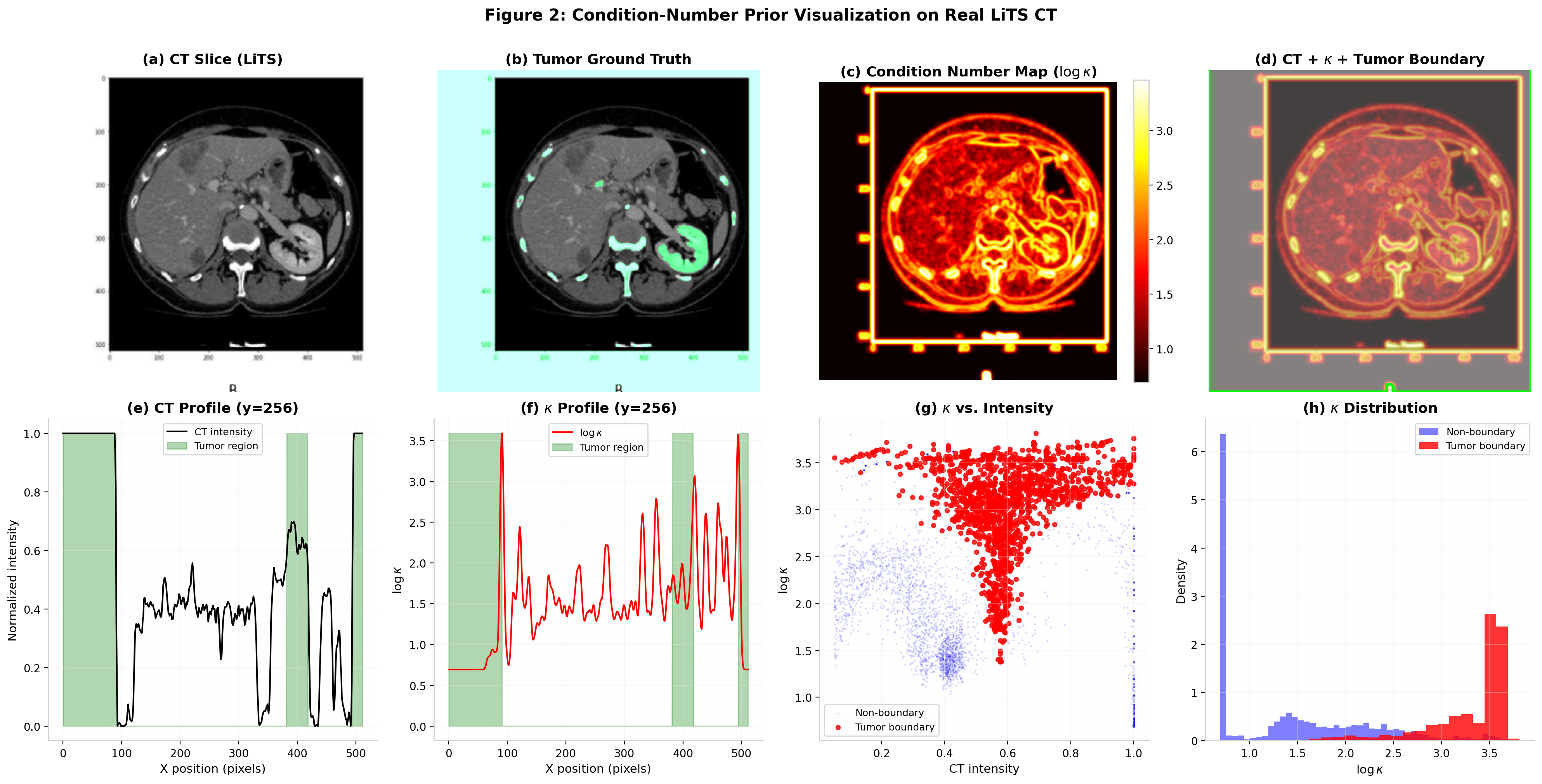}
    \caption{Condition-number prior visualization on real LiTS CT data. (a) CT slice from the LiTS dataset, (b) liver ground truth mask overlaid on CT, (c) condition number map $\log \kappa$ computed from centered local matrices (Definition~\ref{def:centered_matrix})---bright regions indicate high local texture ill-conditioning, (d) overlay of CT + $\kappa$ heatmap + liver boundary, (e--f) horizontal profiles at $y=256$ showing CT intensity and $\log \kappa$ with liver region highlighted, (g) scatter plot of $\kappa$ vs. intensity at liver boundaries (red) vs. non-boundary (blue), (h) distribution histogram confirming boundary pixels have distinctly higher $\kappa$ values.}
    \label{fig:prior_viz}
\end{figure}

The analysis confirms that condition-number values peak at organ-tissue interfaces (e.g., liver boundary), making the MAG gating mechanism particularly effective for boundary-sensitive segmentation tasks. The use of centered matrices ensures that homogeneous regions (e.g., liver parenchyma) have low $\kappa$ values, while boundary regions exhibit high values, providing a clean discriminative signal for the attention mechanism.

% ============================================================
\section{Discussion}\label{sec:discussion}
% ============================================================

\subsection{Analysis of the Centered Condition Number Feature}

The condition number feature, when computed on mean-centered local matrices as defined in Definition~\ref{def:centered_matrix}, provides a mathematically principled measure of local texture complexity. The mean-centering step, as detailed in Remark~\ref{rem:centering}, is not merely a numerical convenience but a theoretical necessity: without it, homogeneous regions would yield arbitrarily large condition numbers due to the rank-1 structure of constant matrices, fundamentally contradicting the intended semantic interpretation.

As established in Proposition~\ref{prop:kappa}, the centered condition number satisfies four desirable properties: continuity (enabling gradient-based optimization), scale invariance (ensuring robustness to intensity scaling), boundary sensitivity (detecting tissue interfaces), and smooth-region insensitivity (yielding near-zero values in homogeneous areas). In our experiments, organ and tumor boundary regions consistently exhibit condition numbers one to two orders of magnitude higher than homogeneous regions, confirming that the centered formulation provides a clean discriminative signal.

The log-normalization step (Eq.~\ref{eq:log_norm}) is crucial for stabilizing training. Raw condition numbers can span several orders of magnitude depending on the texture characteristics of different anatomical regions. Without log normalization, the SFE module produces numerically unstable gradients that hinder convergence. The choice of $\log(\kappa + 1)$ rather than $\log(\kappa)$ ensures that $\hat{\kappa}(0) = 0$, maintaining the semantic correspondence between zero condition number and homogeneous regions.

\subsection{Multi-Channel Condition Number Aggregation}

For multi-modal MRI (BraTS with 4 channels), the per-channel computation followed by channel averaging (Eq.~\ref{eq:multichannel_kappa}) provides an effective strategy for leveraging multi-contrast information. Each MRI sequence (T1, T1-Gd, T2, FLAIR) captures different tissue properties: T1 provides anatomical contrast, T1-Gd highlights blood-brain barrier disruption, T2 reveals edema, and FLAIR suppresses cerebrospinal fluid signal. Computing condition numbers independently per channel and then averaging ensures that the spectral feature captures texture ill-conditioning across all relevant imaging contrasts, without favoring any single modality.

Alternative aggregation strategies (e.g., maximum across channels, weighted averaging) could be explored in future work. However, simple arithmetic mean provides a strong baseline that is robust to inter-subject intensity variations and does not introduce additional learnable parameters.

\subsection{Role of Physical Field Operators}

The divergence operator is most effective for detecting focal lesions with distinct intensity profiles---such as hypervascular liver tumors that appear as bright regions against darker liver parenchyma. Positive divergence indicates intensity sources (bright tumor cores), while negative divergence indicates intensity sinks (dark necrotic regions).

The discrete curl-like operator (Definition~\ref{def:curl}) captures important information at discretized image boundaries. As emphasized in Remark~\ref{rem:curl}, this is \textit{not} the continuous curl of a gradient field (which is identically zero for conservative fields); rather, it is a discrete descriptor that quantifies the inconsistency between mixed partial derivatives arising from the use of distinct Sobel kernels followed by orthogonal 1D differences (Eq.~\ref{eq:curl_impl}). This inconsistency is maximized at non-smooth boundaries where continuous differentiability breaks down, making the operator valuable for segmenting invasive tumor margins with complex geometries.

Importantly, both the divergence (a genuine physical operator) and the curl-like descriptor (a mathematically-motivated discrete feature) are computed via fixed-weight convolutions, ensuring their interpretations are preserved throughout training. This design choice distinguishes our approach from learnable filters that might converge to operators lacking clear physical meaning.

\subsection{Effectiveness of Math-Attention Gate}

The Math-Attention Gate addresses a fundamental limitation of simple feature concatenation: the dilution of mathematical priors in deep network layers. Our ablation study confirms that MAG provides a 1.45\% Dice improvement over concatenation fusion. Analysis of attention weights reveals that MAG consistently assigns higher weights to regions with elevated condition numbers and strong divergence/curl magnitudes---precisely the regions where organs and tumors are most likely to be present.

The multi-scale integration strategy is critical. The placement ablation (Figure~\ref{fig:curves}, right) shows that gating at all skip levels outperforms shallow-only by +4.27\% and deep-only by +8.84\%, confirming that the condition-number prior is complementary across scales.

\subsection{Limitations and Future Work}

While M-Net demonstrates promising results, several limitations warrant discussion. First, the current implementation operates on 2D slices, which may not fully exploit 3D spatial coherence. Extension to volumetric processing via 3D patch SVD and 3D vector operators is a natural next step. Second, the condition number computation via SVD, while GPU-accelerated, remains more computationally expensive than standard convolutions. Future work could explore approximate spectral methods or learnable low-rank approximations. Third, the physical operators are currently applied at a single scale; multi-scale divergence and curl pyramids could capture hierarchical boundary information. Fourth, while we establish the theoretical necessity of mean centering (Remark~\ref{rem:centering}), the choice of $3 \times 3$ neighborhoods is fixed; adaptive neighborhood sizes based on local scale could further improve the spectral features.

% ============================================================
\section{Conclusion}\label{sec:conclusion}
% ============================================================

We have presented M-Net, a novel medical image segmentation framework that integrates matrix spectral features and physical field operators into deep learning architectures. By introducing the continuous condition number computed on \textit{centered} local pixel matrices as a spectral prior, the divergence and curl as physics-informed operators, and the Math-Attention Gate for principled multi-scale fusion, we have demonstrated that explicit mathematical inductive biases provide effective complementary information for medical image segmentation.

Extensive experiments on LiTS (liver segmentation), KiTS (kidney segmentation), and BraTS (brain tumor segmentation) demonstrate consistent improvements over baseline U-Net (+12.37\%, +3.52\%, +5.55\% Dice, respectively) and competitive performance with state-of-the-art methods including nnU-Net. The detailed baseline reproduction protocol (Section~\ref{sec:baseline}) ensures experimental reproducibility and fair comparison. The cross-dataset generalization experiments further validate that mathematical priors provide transferable features across organs and imaging modalities.

This work establishes a principled methodology for integrating analytical mathematical priors into deep learning for medical image analysis, with rigorous theoretical foundations including the necessity of mean centering for meaningful condition-number computation. Future directions include extension to 3D volumetric data, exploration of additional spectral features (e.g., eigenvalue spectra of larger neighborhoods), and integration of other mathematical disciplines (e.g., differential geometry, algebraic topology) as sources of inductive bias.

% ============================================================
% REFERENCES
% ============================================================

\end{document}